\documentclass[11pt]{article}
\usepackage{fullpage}
\usepackage[numbers]{natbib}
\usepackage{hyperref}
\usepackage[svgnames]{xcolor}
\usepackage{overpic}
\usepackage{tikz}
\usepackage{rotating}

\hypersetup{colorlinks=true,citecolor=blue,linkcolor=MediumSlateBlue, urlcolor=DarkBlue}

\usepackage{macros}
\usepackage{color}
\usepackage{bbm}
\usepackage{nicefrac}
\usepackage{psfrag}
\usepackage{xifthen} 
\usepackage{xr}
\usepackage{amsmath}
\usepackage{amsthm}
\usepackage{float}
\newtheorem{prop}{Proposition}
\usepackage{algorithm}
\usepackage{algorithmic}
\usepackage{graphicx} 
\usepackage{subfigure} 
\def\abovestrut#1{\rule[0in]{0in}{#1}\ignorespaces}
\def\belowstrut#1{\rule[-#1]{0in}{#1}\ignorespaces}

\def\abovespace{\abovestrut{0.20in}}

\def\belowspace{\belowstrut{0.10in}}

\makeatletter
\long\def\@makecaption#1#2{
  \vskip 0.8ex
  \setbox\@tempboxa\hbox{\small {\bf #1:} #2}
  \parindent 1.5em  
  \dimen0=\hsize
  \advance\dimen0 by -3em
  \ifdim \wd\@tempboxa >\dimen0
  \hbox to \hsize{
    \parindent 0em
    \hfil 
    \parbox{\dimen0}{\def\baselinestretch{0.96}\small
      {\bf #1.} #2
    } 
    \hfil}
  \else \hbox to \hsize{\hfil \box\@tempboxa \hfil}
  \fi
}
\makeatother

\long\def\comment#1{}

\begin{document}

\begin{center}
  {\Large{\bf{Feature Selection Facilitates Learning Mixtures of Discrete Product Distributions}}}

  \vspace*{.3in}
  
  \begin{tabular}{ccc}
    Vincent\ Zhao & Steven W.\ Zucker 
    \\ \href{mailto:yuzhe.zhao@yale.edu}{yuzhe.zhao@yale.edu} &
    \href{mailto:zucker@cs.yale.edu}{zucker@cs.yale.edu}
  \end{tabular}
  
  \vspace*{.5cm}
  
  \begin{tabular}{ccc}
    Yale University &&
    \\
    New Haven, CT 06511 
  \end{tabular}
  
  \vspace*{.2in}
\end{center}

\begin{abstract}  
  Feature selection can facilitate the learning of mixtures of discrete random variables as they 
arise, e.g. in crowdsourcing tasks. Intuitively, not all workers are equally reliable but, if the less reliable ones could
be eliminated, then learning should be more robust. By analogy with Gaussian mixture models, we seek a
low-order statistical approach, and here introduce an algorithm based on the (pairwise) mutual information.
This induces an order over workers that is well structured for the `one coin' model. More generally, it is justified by a goodness-of-fit measure and is validated empirically. Improvement in real data sets can be substantial.
\end{abstract}

\section{Introduction}

Mixtures of discrete product distributions (MDPD) have been applied widely to large problems, from computational neuroscience to bioinformatics and recommendation systems. Here we concentrate on another, popular one -- crowdsourcing \cite{dawid1979maximum} -- to introduce a feature selection algorithm that works in the discrete variable setting. In effect, our algorithm enhances the learning process by identifying those workers who are likely to be performing well. We show experimentally that this feature selection leads to better performance than other state-of-the art algorithms, and we provide a theoretical framework that suggests why this is to be expected. 

Learning MDPD is NP-hard. Although many authors have proposed algorithms and heuristics to learn MDPD under different circumstances \cite{dawid1979maximum,zhang2014spectral,jain2014learning,feldman2008learning}, there is almost no literature concerning the feature selection problem as we formulate it.  An exception is \cite{chaudhuri2008learning}, who sought features that sharply separate mixture components. Their algorithm is based on correlations of the input data, but is restricted to mixtures of binary product distributions; our algorithm is applicable to general MDPD and is based on pairwise mutual information. Another group sought to identify reliable workers directly \cite{ghosh2011moderates,whitehill2009whose},  but this led to algorithms that are specific to crowdsourcing and hard to generalize for MDPD.

Dimensionality reduction for Gaussian mixture models is better studied. \cite{pan2007penalized} proposed an algorithm based on a penalized likelihood function that leads to an EM variant with a regularized M-step. \cite{azizyan2013minimax} analyze learning for a mixture of two isotropic Gaussians in high dimensions under sparse mean separation. More recently, \cite{jin2016influential} proposed an algorithm to discover influential features for high dimensional clustering. The dimensionality reduction methods in \cite{,azizyan2013minimax, jin2016influential} are based on Principal Component Analysis, which constructs features that are linear combinations of the input variables. This underlines the fundamental difference between the continuous- and the discrete-valued problems: linear combination is not a valid operator for discrete random variables. To see this, let $X$ denote a random variable which takes value from $\{`\alpha', `\beta'\}$, while $Y \in \{`a', `b', `c'\}$ is another discrete variable. $X+Y$ is obviously not well-defined.

Even though a direct generalization of the techniques for Gaussian mixture models to MDPD is not proper, the continuous variable case has been a source of inspiration in the following sense. PCA performs an eigen-decomposition on the sample covariance matrix which relies, in turn, on second-order statistics of the data. The second-order statistics for discrete random variables are basically co-occurrence. Thus we ask: can dimensionality reduction be based on co-occurrence for MDPD, which forms an analogue to the use of PCA for Gaussian mixture models? We give a positive answer in this paper and propose a novel feature selection technique for MDPD that is based on pairwise mutual information. 
The utilization of pairwise mutual information is justified by its connection to a goodness-of-fit measure and is validated by empirical studies on real crowdsourcing datasets. We show that, in effect,  the algorithm filters out noise and makes the learning more robust; in many cases we significantly reduce the error rates.


\section{Background}
\label{sec:backgroud}
We study mixtures of discrete product distributions (MDPD). Throughout the paper, we use the uppercase letters $X$ and $Y$ for random variables and the lowercase letters $x$ and $y$ for their instances (realizations). Let $X_i$ be an observable discrete variable and $Y$ the latent variable. $X_i$ takes discrete values $X_i \in \{1, 2, \ldots, C\}$ and $Y$ indicates the mixture component $Y \in \{1,2, \dots, K\}$, where $K$ is the total number of components. Let $i = 1, 2, \ldots, p$, $p$ is the dimension of the model and $X = [X_1, X_2, \ldots, X_p]^T$. MDPD is a generative model with joint probability distribution:
\begin{equation*}
  P(X, Y) = \sum_{k=1}^K P(Y=k) P(X|Y=k).
\end{equation*}
$P(X|Y=k)$ is a product distribution, i.e. 
\begin{equation*}
  P(X|Y=k) = \prod_{i=1}^{p} P(X_i|Y=k).
\end{equation*}
Given the observations $\{x_i^{(n)}\}$, the goal is to estimate the model parameters, i.e. $p(Y=k)$ and $p(X_i|Y=k)$, for all $i$ and $k$. There are many papers addressing this learning problem \cite{,dawid1979maximum,zhang2014spectral,jain2014learning,chaudhuri2008learning}. In general, those algorithms can be classified into two groups, (1) maximum likelihood estimation and (2) method of moments. The EM algorithm and its variants have been widely used to maximize the log-likelihood. However, since the log-likelihood function is non-convex, these algorithms can be stuck in a bad local maximum. Recently, several authors \cite{zhang2014spectral,jain2014learning} proposed algorithms based on method of moments for learning MDPD which relies on third-order moments. The performance of these algorithms is statistically provable under certain conditions. 


\section{Feature Selection for MDPD}
\label{sec:method}
The problem of feature selection is to reduce the model dimension by identifying a useful and relevant feature subset. It is used to simplify the model for easier interpretation, to reduce training time, to overcome the curse of dimensionality and to avoid over-fitting thereby making the model more robust. Most literature on feature selection focuses on supervised learning, where the usefulness and relevance of features are generally defined by their prediction power. Feature selection and dimensionality reduction for unsupervised learning are more challenging problems, due to the lack of labeled data. Refer to \cite{guyon2003introduction,dy2004feature} for reviews on this topic. 

It is well-known that the EM algorithm is sensitive to initialization, while the method of moments \cite{zhang2014spectral} is sensitive to some global properties of the model. The performance of both algorithms can be dramatically impaired by noisy, irrelevant and redundant data. This makes feature selection relevant to learning MDPD; and critical in practice. In this section, we introduce our feature selection technique based on pairwise mutual information and illustrate the underlying ideas. 

Intuitively, we want to identify those features that are discriminative of the latent variable $Y$, despite the lack of any direct access to that latent variable. Nothing can be said about $Y$ if only one observable variable is revealed, because the one-dimensional MDPD is not identifiable. Therefore, the learning algorithm has to rely on the interaction among different observable variables. For MDPD, if $X_i$ is known to be independent of $Y$, it can be shown that $X_i$ must also be independent of $X_j$ ($j\neq i$). On the other hand, if a strong dependence between $X_i$ and $X_j$ is observed, it can be concluded that $X_i$ and $X_j$ are discriminative of $Y$ and should be identified as useful features.

Our feature selection technique is motivated by the argument above. We use mutual information to measure the dependence between two variables,
\begin{equation*}
  I(X,Y) = \sum_{x,y} P(x,y) \log\frac{P(x,y)}{P(x)P(y)}.
\end{equation*}

\begin{algorithm}[tb]
   \caption{Feature Selection for MDPD}
   \label{algo}
\begin{algorithmic}
   \STATE {\bfseries Input:} the number of features to be selected $L$, observed data $x_i^{(n)}$ for $i = 1,2,\ldots,p$ and $n = 1, 2, \dots, N$.
   \STATE {Estimate $I(X_i, X_j)$ from the data.}
   \STATE {Use either of the two heuristics:}
   \STATE{(a) Find the feature subset $S$ of size $L$ so that 
   \begin{equation*}
    S = \arg\max_S \sum_{\substack{i<j\\i,j\in S}} I(X_i, X_j).
   \end{equation*}}
   \STATE{(b)} For each $X_i$, calculate the mutual information $\text{score}_i = \sum_{j \neq i} I(X_i, X_j)$ and select top $L$ features according to their scores. 
\end{algorithmic}
\end{algorithm}

The feature selection technique is shown in algorithm \ref{algo}. First, we estimate the joint probability $P(x_i, x_j)$ with $\hat{P}(x_i, x_j) = \frac{\#(x_i, x_j)}{N}$. $\#(x_i,x_j)$ is the co-occurrence between $x_i$ and $x_j$ and $N$ is sample size. Then, we estimate pairwise mutual information $I(X_i, X_j)$ with $\hat{P}(x_i, x_j)$, for all $i$ and $j$. After getting the pairwise mutual information matrix $[\hat{I}(X_i, X_j)]_{p\times p}$, two feature selection heuristics are proposed. The first one is to maximize the sum of the entries of sub-matrices of the pairwise mutual information matrix. The other one is based on feature ranking according to the mutual information score, i.e.
\begin{equation}
  \text{score}_i = \sum_{j \neq i} I(X_i, X_j).
  \label{eq:score}
\end{equation}

In practice, the mutual information score can be used to decide the number of features to be used in the model. In section \ref{sec:experiments}, we plot the mutual information score for the features in the real datasets. It is observed that the score drops quickly after the top few features and has a relatively flat tail. The curve ``resembles'' the plot of eigenvalues of PCA. Therefore, we may set a cut-off according to the gradient of the curve.

\subsection{One-Coin Model}
To demonstrate the feature selection technique, we consider a simple mixture of discrete product distributions that is usually referred to as ``one-coin model'' in crowdsourcing. For one-coin model, the number of components $K$ is identical to $C$. We assume that $Y$ is uniformly distributed, i.e. $p(Y=k) = \frac{1}{K}$ for $k = 1,2,\ldots, K$. And the conditional probability of $X_i$ is parameterized by a single parameter $p_i$. More concretely, it is defined as 
\begin{equation}
  P(X_i = c|Y=k) =
  \begin{cases}
    p_i & \text{if } c=k \\
    \frac{1-p_i}{k-1} & \text{if } c \neq k
  \end{cases}
\end{equation}

In other words, the worker $i$ uses a single coin flip to decide the label. With probability $p_i$, the worker gives the correct label, whatever the true label is. And with probability $1-p_i$, he randomly gives an incorrect label. In this case, it is intuitive to define the capabilities of workers. A worker with larger $p_i$ is more capable than a worker with smaller $p_i$. A worker with $p_i = 1$ is the best, because he always gives the correct label. Given a group of workers with different capabilities, the goal of feature selection is to find the those most capable ones. 

The mutual information $I(X_i,X_j)$ depends on the joint probability distribution $P(X_i,X_j)$. Since $P(X_i = c) =  \frac{1}{K}$, the marginal distribution of $X_i$ is uniform for all $i$. The joint distribution $P(X_i,X_j)$ can be represented by a $C$-by-$C$ symmetric matrix whose diagonal elements and off-diagonal elements are (respectively) identical. Let the diagonal elements $p(X_i=c, X_j=c)$ be denoted $\alpha$ and those off-diagonal elements $p(X_i=b, x_j=c)$ ($b\neq c$) become $\frac{1 - K \alpha}{K(K-1)}$. The mutual information is then
\begin{equation}
  I(X_i, X_j) = K \alpha \log{\alpha} + (1-K\alpha)\log{\frac{1-K\alpha}{K(K-1)}} + \log(K^2). \\
\end{equation}
It equals zero when $\alpha = \frac{1}{K^2}$ and is monotonically increasing when $\alpha > \frac{1}{K^2}$. In addition, $\alpha$ can be expressed as a function of $p_i$ and $p_j$, i.e.
\begin{align*}
  \alpha &= \frac{1}{K}p_ip_j + \frac{K-1}{K}\frac{1-p_i}{K-1}\frac{1-p_j}{K-1} \\
  &= \frac{1}{4(K-1)}[(p_i + p_j - \frac{2}{K})^2 - (p_i - p_j)^2] + \frac{1}{K^2}.
\end{align*}

This function describes a hyperbolic paraboloid as shown in Figure \ref{fig:saddle}. When $p_i = p_j = \frac{1}{K}$, the function is at its saddle point, where $\alpha = \frac{1}{K^2}$. In the region where both $p_i$ and $p_j$ are larger than $\frac{1}{K}$, $\alpha$ is monotonically increasing with regard to $p_i$ when $p_j$ is fixed, and vice versa. 

\begin{figure}
\centering
\begin{minipage}{.49\textwidth}
	\centering
	\includegraphics[width=\columnwidth]{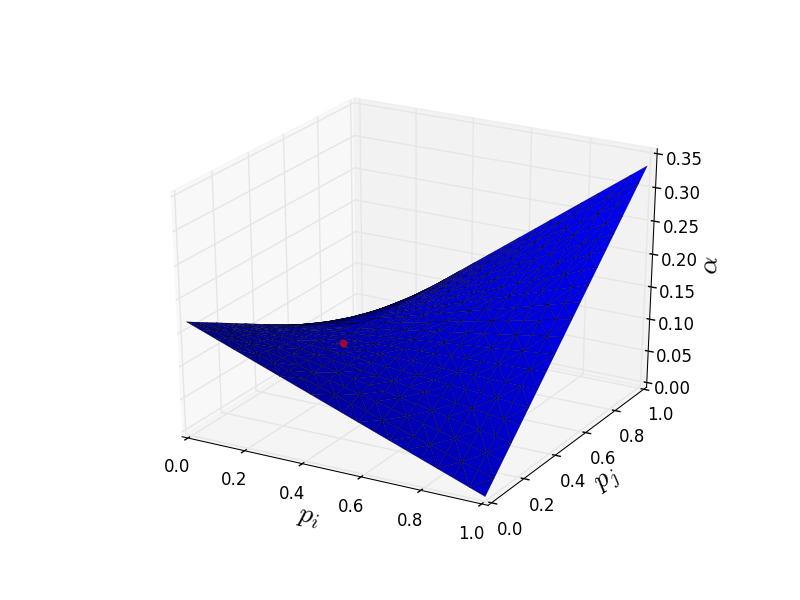}
	\caption{The figure shows $\alpha$ as a function of $p_i$ and $p_j$ when $K=3$. The red dot is the saddle point of the hyperbolic paraboloid, where $p_i = p_j = \frac{1}{3}$ and $\alpha = \frac{1}{9}$. In the area when $p_i, p_j > \frac{1}{3}$, $\alpha$ is monotonically increasing with regard to either $p_i$ or $p_j$ when the other is fixed.}
	\label{fig:saddle}
\end{minipage}
\begin{minipage}{.49\textwidth}
	\centering
	\includegraphics[width=.85\columnwidth]{illustration}
	\caption{This figure shows the relationship between the proposed goodness-of-fit measure $G(\Theta^{t+1}; \Theta^t)$ and the KL-divergence $D_{KL}(P_0(X) || P_{\Theta^t}(X))$. The two curves in the figure are the marginal log-likelihood and its lower bound derived from Jansen's inequality.}
	\label{fig:illustration}
\end{minipage}

\end{figure}

Thus, we conclude that when $p_i > \frac{1}{K}$ and $p_j > \frac{1}{K}$ (i.e. when workers are better than guess randomly.), the mutual information $I(X_i, X_j)$ is monotonically increasing with regard to $p_i$ and $p_j$. In other words, if worker $i$ is more capable than worker $j$ (i.e. $p_i > p_j > \frac{1}{K}$), we have $I(X_i, X_k) > I(X_j, X_k)$. This is enough to guarantee that our feature selection techniques (either (a) or (b) in algorithm \ref{algo}) will always select $X_i$ over $X_j$.

\section{Pairwise Mutual Information and Maximum Likelihood Estimation}
\label{sec:theory}

In this section, the use of pairwise mutual information is justified with theoretical analysis that reveals its relation to maximum likelihood estimation and a goodness-of-fit measure. We start by introducing the maximum likelihood objective function of MDPD and the well-known expectation-maximization (EM) algorithm. Provided $N$ data points, MLE seeks to maximize the marginal log-likelihood
\begin{equation*}
  l(\Theta) := \sum_{n=1}^N \log P_\Theta(x^{(n)})
\end{equation*}
where $P_\Theta(x) = \sum_y P_\Theta(y) \prod_i P_\Theta(x_i|y)$. $\Theta$ denotes all the parameters of the model, i.e. $\omega_k = P_\Theta(Y=k)$ and $\mu_{irk} = P_\Theta(X_i=r|Y=k)$. This is the standard definition of log-likelihood with finite samples. However, for convenience, we conduct our analysis at the population level (infinite sample size). Let $P_0(X)$ denote the underlying distribution from which samples are drawn. The marginal likelihood can be defined as
\begin{equation*}
  l(\Theta) := \sum_x P_0(x) \log P_\Theta(x).
 \end{equation*} 

Direct optimization on the marginal likelihood is hard. A common workaround uses Jensen's inequality to relax the problem. Let $q(Y)$ denote a probability distribution over $Y$. By applying Jensen's inequality, we have
\begin{align*}
  l(\Theta) &=  \sum_x P_0(x) \log \sum_y P_\Theta(x, y) \\
  &\geq \sum_x P_0(x) \sum_y q(y) \log \frac{P_\Theta(x, y)}{q(y)}.
\end{align*}
Instead of maximizing the marginal log-likelihood, we are going to maximize the function on the right-hand side. This leads to the EM algorithm, an iterative algorithm consisting of two steps. Let $\Theta^t$ be the model parameters at time $t$.

\textbf{E-step}: Calculate the posterior distribution $P_{\Theta^t}(Y|x)$ for all the configurations of $x$ and let $q(y)$ be $P_{\Theta^t}(y|x)$.

\textbf{M-step}: Update the parameters by calculating 
\begin{equation}
  \Theta^{t+1} = \arg\max_\Theta F(\Theta; \Theta^t)
  \label{eq:Mstep}
\end{equation}
where $F(\Theta; \Theta^t) = \sum_x P_0(x) \sum_y P_{\Theta^t}(y|x) \log \frac{P_\Theta(x, y)}{P_{\Theta^t}(y|x)}$.

On the other hand, we want to get an upper bound of the log-likelihood. The KL-divergence $D_{KL}\left(P_0(X) || P_\Theta(X)\right)$ is defined as

\begin{align*}
  D_{KL}(P_0(X) || P_\Theta(X)) &:= \sum P_0(x) \log \frac{P_0(x)}{P_\Theta(x)} \\
  &= - H_{P_0}(X) - l(\Theta).
\end{align*}

It equals the difference between the negative entropy of the data and the marginal log-likelihood. Due to the non-negativity of KL-divergence, the marginal log-likelihood $l(\Theta)$ is upper bounded by the negative entropy $- H_{P_0}(X)$. Moreover, the KL-divergence equals zero when $P_{\Theta}(x)$ and $P_0(x)$ are identical almost everywhere. Thus, the KL-divergence can be considered as a goodness-of-fit measure for mixture models. However, we usually don't have access to the probability distribution $P_0(X)$ and estimating the negative entropy from the data is computationally intractable. To overcome the difficulty, we consider using $F(\Theta; \Theta^t)$ to approximate $l(\Theta)$. 

\begin{equation}
  G(\Theta; \Theta^t) := - H_{P_0}(X) - F(\Theta; \Theta^t)
  \label{eq:G}
\end{equation}

$G(\Theta; \Theta^t)$ is defined to be the difference between $- H_{P_0}(X)$ and $F(\Theta; \Theta^t)$. $G(\Theta; \Theta^t)$ is a function of $\Theta$ with parameter $\Theta^t$. And equation \ref{eq:Mstep} leads to fact that $G(\Theta^{t+1}; \Theta^t)= \min_\Theta G(\Theta; \Theta^t)$. Later on, we will focus on $G(\Theta^{t+1}; \Theta^t)$. It can be shown that $G(\Theta^{t+1}; \Theta^t)$ underestimates $D_{KL}(P_0(X) || P_\Theta^t(X))$ but overestimates $D_{KL}(P_0(X) || P_\Theta^{t+1}(X))$. The relation between $G(\Theta^{t+1}; \Theta^t)$ and $D_{KL}(P_0(X) || P_{\Theta^t}(X))$ is illustrated in figure \ref{fig:illustration}. Moreover, we have the following lemma.

\begin{lem}
\label{lem: multi information}
Let $G(\Theta; \Theta^t)$ be defined as equation \ref{eq:G} and $\Theta^{t+1} = \arg\min_\Theta G(\Theta; \Theta^t)$. For MDPD, it can be shown that 
\begin{equation}
  G(\Theta^{t+1}; \Theta^t) = \sum_x \sum_y \tilde{P}_{\Theta^t}(x,y) \log \frac{\tilde{P}_{\Theta^t}(x|y)}{\prod_i \tilde{P}_{\Theta^t}(x_i|y)}.
\end{equation}
where $\tilde{P}_{\Theta^t}(x,y) := P_0(x)P_{\Theta^t} (y|x)$.
\end{lem}
\begin{proof}
 By the definition of $G(\Theta; \Theta^t)$, to minimize $G(\Theta; \Theta^t)$ is equivalent to maximize $F(\Theta;\Theta^t)$, which is basically the M-step in EM. Since $P_\Theta(x,y) = P_\Theta(y) \prod_i P_\Theta(x_i|y)$, it is straightforward from equation \ref{eq:Mstep} that 
 \begin{align*}
  P_{\Theta^{t+1}}(Y) &= \tilde{P}_{\Theta^t} (Y) \\
  P_{\Theta^{t+1}}(X|Y) &= \tilde{P}_{\Theta^t}(X|Y).
 \end{align*}
Therefore,
\begin{align}
  G(\Theta^{t+1}; \Theta^t) &= - H_{P_0}(X) - F(\Theta^{t+1}; \Theta^t) \\
  &= \sum_{x,y} P_0(x) P_{\Theta^t}(y|x) \log \frac{P_0(x)P_{\Theta^t}(y|x)}{P_{\Theta^{t+1}}(x,y)} \\
  &= \sum_{x,y} P_0(x) P_{\Theta^t}(y|x) \log \frac{\tilde{P}_{\Theta^t}(x, y)}{\tilde{P}_{\Theta^t}(y)\prod_i \tilde{P}_{\Theta^t}(x_i|y)}
\end{align}
\end{proof}

In information theory, the multi-information of a multivariate probabilistic distribution $p(X)$ is defined as
\begin{equation*}
  \sum_x P(x) \log \frac{P(x)}{\prod_i P(x_i)}.
 \end{equation*}
It is the KL-divergence between $p(X)$ and the product distribution $\prod_i p(X_i)$. Multi-information is zero when the random variables are mutually independent. According to lemma \ref{lem: multi information}, $G(\Theta^{t+1}; \Theta^t)$ measures the dependency among variables left in the data which is not explained by the current mixture model. It seems promising, however it is still computational intractable. As a work-around, we apply Bethe entropy approximation \cite{wainwright2008graphical} to approximate multi-information with the sum of pairwise mutual information. This leads to an approximated goodness-of-fit measure (equation \ref{eq:approx}) for MDPD which only relies on the second-order statistics of the data; it can be calculated efficiently.
\begin{equation}
  G(\Theta^{t+1}; \Theta^t) \approx \sum_{i<j} I_{\tilde{P}_{\Theta^t}}(X_i, X_j|Y)
  \label{eq:approx}
\end{equation}
where the conditional mutual information 
\begin{equation*}
  I_{\tilde{P}_{\Theta^t}}(X_i, X_j|Y) = \sum_{x_i, x_j, y} \tilde{P}_{\Theta^t}(x_i,x_j,y) \log\frac{\tilde{P}_{\Theta^t}(x_i,x_j|y)}{\tilde{P}_{\Theta^t}(x_j|y)\tilde{P}_{\Theta^t}(x_i|y)}.
\end{equation*}

To summarize, we have derived a goodness-of-fit measure (equation \ref{eq:approx}) for MDPD based on maximum likelihood estimation and information theory. The question is how it is related to the feature selection algorithm we have proposed earlier.

\begin{prop}
  Let $P_0(X)$ be the underlying probability distribution of the data and $P_{\Theta^0}(X, Y)$ be an one-component mixture model satisfying $P_{\Theta^0}(X_i|Y=1) = P_0(X_i)$. Therefore, the proposed goodness-of-fit measure (equation \ref{eq:approx}) becomes
  \begin{equation}
    \sum_{i< j} I_{P_0}(X_i, X_j)
    \label{eq:sum MI}
  \end{equation}
\end{prop}
\begin{proof}
  The proof follows the fact that since there is only one mixture component, we always have $P_{\Theta^0}(y|x) = 1$ and it leads to $\tilde{P}_{\Theta^0}(x, y) = P_0(x)$.
\end{proof}

This proposition indicates that the sum of pairwise mutual information (equation \ref{eq:sum MI}), which can be estimated from data, is actually a goodness-of-fit measure of the one-component mixture model. If the features are mutually independent, an one-component mixture model will be enough to model the data perfectly and the sum of pairwise mutual information will be close to zero. Our feature selection algorithm (algorithm \ref{algo}) selects the feature subset that maximizes the sum of mutual information with regard to the feature set. In other words, the selected features are the dimensions where the one-component mixture model doesn't explain the data well.

\section{Empirical Studies}
\label{sec:experiments}
In this section, we demonstrate our feature selection algorithm for crowdsourcing. Crowdsourcing has been an popular way to collect labels for large datasets in many application domains, including computer vision and natural language processing. Web services such as Amazon Mechanical Turk provide platforms where human intelligence tasks are posted and large quantities of labels from hundreds of online workers are collected. The problem is to infer the true labels for datasets from the collected labels.

The performances of different algorithms under our feature selection method are compared. And five real datasets are used in this study. We show that the algorithms are able to achieve a low mis-clustering rate with fairly small feature (worker) subsets, which reveals the redundancy inherent in the real datasets. In some cases, feature selection can even significantly boost the performance. 

We also compare our feature selection algorithm to a supervised feature selection method. The supervised feature selection is done by ranking features according to their individual mis-clustering rate and selecting top features accordingly. As the real problem is essentially unsupervised, using a supervised feature selection is `cheating', as it leaks true labels to the algorithm. Nevertheless, it provides a benchmark of how useful feature selection could possibly be.

\subsection{Spectral Method and Majority Voting}
According to \cite{zhang2014spectral} and related papers, spectral method (\emph{opt-D\&S}) and majority voting (with EM) outperforms other algorithms on these datasets. Therefore, we implement these two algorithms in our study.

The spectral method is a two-stage algorithm proposed in \cite{zhang2014spectral}. The first stage uses the method of moments and tensor decomposition to estimate the mixture model parameters, while the second stage runs regular EM iterations taking the results of the first stage as initialization. The first stage of the algorithm randomly partitions all the workers into three disjoint groups. Therefore, the performance of the algorithm may fluctuate. To properly evaluate the performance, we repeat the spectral method multiple times and report the median, the first, and the third quartile. 

Majority voting is a simple and popular algorithm for crowdsourcing. It gives the prediction by summing up all worker labels and picks the one with the highest votes. When there are ties in the votes, it randomly picks one and we report the expected mis-clustering error. For example, if the votes for three labels are tied, the expected error will be $\frac{2}{3}$. When we evaluate mis-clustering rate, due to missing values, it is possible that some items receive no votes from the selected workers. In those cases, we treat them as ties.

\subsection{Real Datasets and Deal with Missing Values}

\begin{table}[t]
\caption{The summary of datasets used in the empirical study.}
\label{tb:data summary}
\vskip 0.15in
\begin{center}
\begin{small}
\begin{sc}
\begin{tabular}{p{.8cm}cccp{1.5cm}}
\hline
\abovespace\belowspace
Data-sets & \# classes & \# items & \# workers & \# worker labels\\
\hline
\abovespace
Bird    & 2 & 108 & 39 & 4,212 \\
RTE     & 2 & 800 & 164 & 8,000  \\
TREC    & 2 & 19,033 & 762 & 88,385  \\
Dog     & 4 & 807 & 109 & 8,070 \\
Web     & 5 & 2,665 & 177 & 15,567 \\
\hline
\belowspace
\end{tabular}
\end{sc}
\end{small}
\end{center}
\vskip -0.1in
\end{table}

Five real crowdsourcing data sets are used in this study: (1) bird dataset \cite{welinder2010multidimensional} is a binary labeling task , (2) recognizing textual entailment (RTE) dataset \cite{snow2008cheap} contains pairs of sentences and is a binary task to determine if the second sentence can be inferred from the first, (3) TREC is a binary task from TREC 2011 crowdsourcing track \cite{lease2011overview} assessing the quality of information retrieval, (4) Dog dataset contains a set of pictures from ImageNet \cite{deng2009imagenet} and the task is to label the four breads of dogs, (5) web dataset \cite{zhou2012learning} is a set of query-URL pairs for workers to label a relevance score from 1 to 5.

Except for the bird dataset, the other datasets contain lots of missing values. It is common for real datasets, as workers do not assign labels to all the items. To accommodate our feature selection technique to missing values, a natural way is to add a virtual label for each variable $X_i$, i.e. $X_i \in \{1, 2, \ldots, C, \text{`n/a'}\}$. If we assume that $X_i$ being missing is not discriminative of the latent variable $Y$, we can adjust the algorithm by calculating $\sum_{\substack{x_i\neq `n/a' \\ x_j\neq`n/a'}} P(x_i,x_j) \log\frac{P(x_i,x_j)}{P(x_i)P(x_j)}$ for $I(X_i, X_j)$, to eliminate the contribution of the virtual label to mutual information.


\subsection{Results}

\begin{table}[t]
\caption{Mis-clustering rate (\%) of algorithms are reported. For \emph{opt-D\&S} \cite{zhang2014spectral}, we repeated the algorithm 20 times and report the median error rate. The top rows show the results of the algorithms on the complete datasets and the bottom rows demonstrate the results after feature selection. The numbers in the parentheses are the number of features used when the algorithms achieve the optimal accuracy.}
\label{tb:result}
\vskip 0.15in
\begin{center}
\begin{small}
\begin{sc}
\begin{tabular}{lccr}
\hline
\abovespace\belowspace
Dataset & Opt-D\&S & MV & MV+EM \\
\hline
\abovespace
Bird    & 11.11 & 24.07 & \textbf{10.18}  \\
RTE     & \textbf{7.12} & 10.31 & 7.25  \\
TREC    & 32.33 & 34.86 & \textbf{29.76}   \\
Dog     & 15.75 & 17.78 & \textbf{15.74}  \\
Web     & 29.22 & 27.09 & \textbf{17.52}  \\
\hline
\hline
\abovespace\belowspace
& FS+Opt-D\&S & FS+MV & FS+MV+EM \\
\hline
\abovespace
Bird    & \textbf{8.33} (15) & 10.18 (5) & \textbf{8.33} (15) \\
RTE     & \textbf{7.12} (163) & 8.00 (162) & 7.25 (159) \\
TREC    & 30.11 (425) & 34.81 (378) & \textbf{29.47} (459) \\
Dog     & \textbf{15.46} (76) & 17.35 (64) & 15.49 (75) \\
Web     & 11.41 (17) & 12.03 (8) & \textbf{11.20} (9) \\
\hline

\end{tabular}
\end{sc}
\end{small}
\end{center}
\vskip -0.1in
\end{table}

We report the mis-clustering rate of different algorithms and their performance under feature selection in table \ref{tb:result}. Majority voting (alone) are probably thought as the simplest algorithm for crowdsourcing. As known to the crowdsourcing society, using the EM to refine the majority voting algorithm can improve the error rate (see the top half of the table). This is probably due to the noise of the worker labels. We show that with proper feature (worker) selection, the noise can be reduced. For example, for bird and web datasets, majority voting did not work well on the complete datasets, compared to \emph{opt-D\&S} and \emph{MV+EM}. However, after feature selection, the performance of majority voting becomes on a par with or even better than the performances of the more sophisticated algorithms (without feature selection). Also, both \emph{opt-D\&S} and \emph{MV+EM} benefit from feature selection in terms of the mis-clustering rate. Moreover, the results shed light on the redundant nature of crowdsourcing datasets.



To better understand the influence of our feature selection technique, figure \ref{fig:datasets} show the mis-clustering rates of the algorithms at different levels of feature selection. 

\begin{figure}
\centering
\begin{minipage}{.49\columnwidth}
	\centerline{\includegraphics[width=\columnwidth]{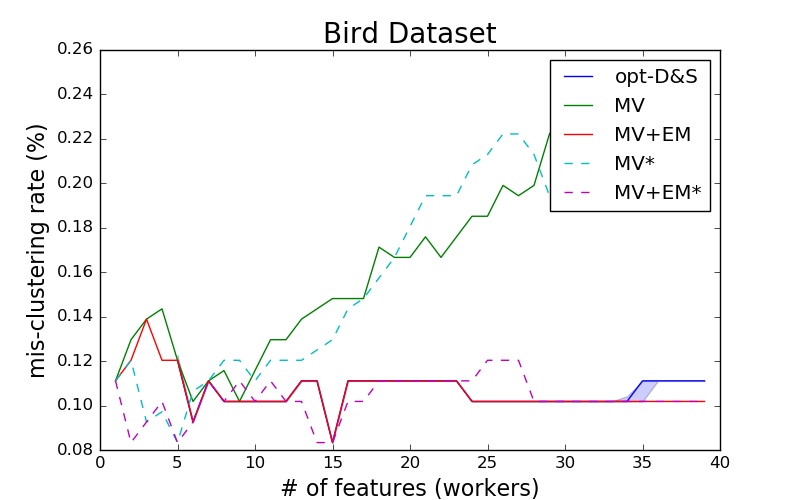}}
	\centerline{\includegraphics[width=\columnwidth]{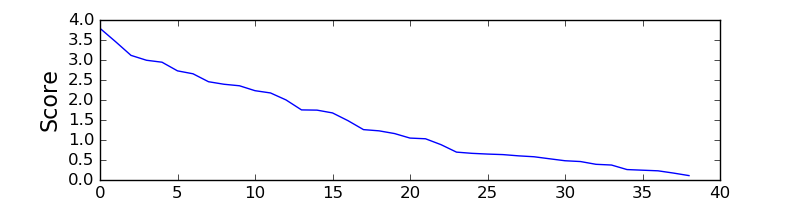}}	
\end{minipage}
\begin{minipage}{.49\columnwidth}
	\centerline{\includegraphics[width=\columnwidth]{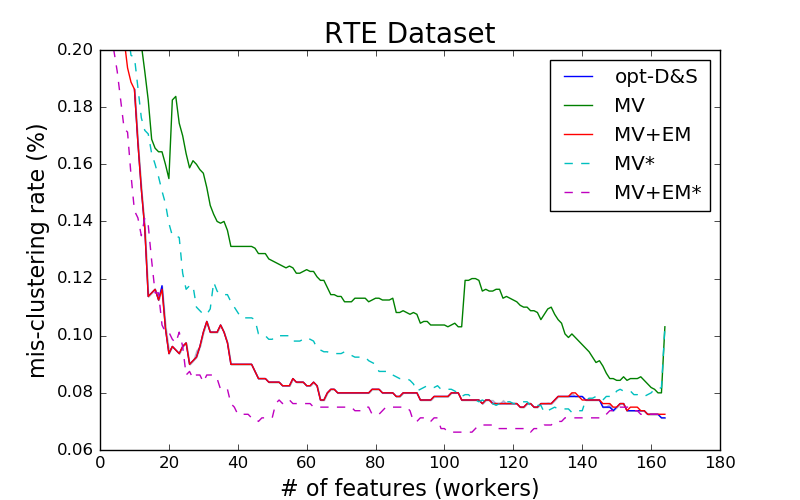}}
	\centerline{\includegraphics[width=\columnwidth]{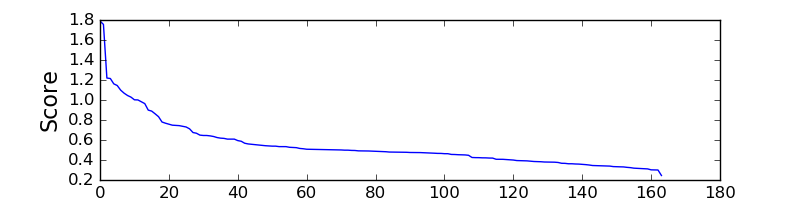}}
\end{minipage}
\begin{minipage}{.49\columnwidth}
	\centerline{\includegraphics[width=\columnwidth]{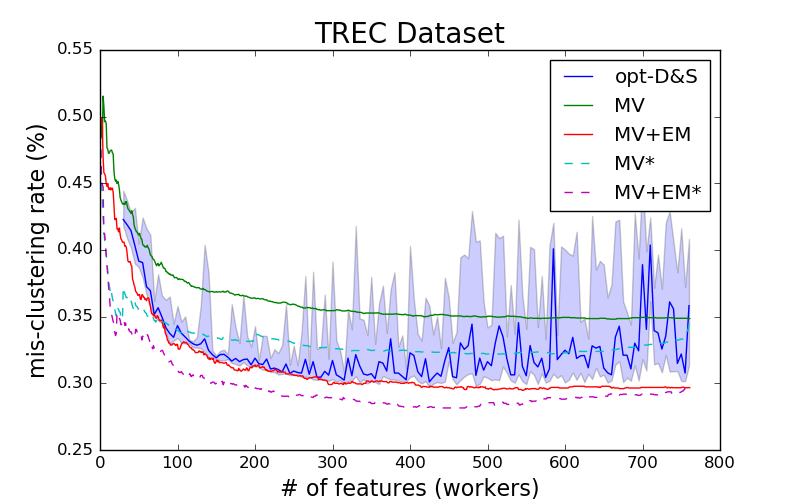}}
	\centerline{\includegraphics[width=\columnwidth]{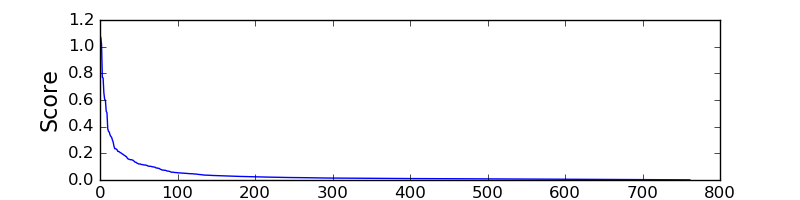}}
\end{minipage}
\begin{minipage}{.49\columnwidth}
	\centerline{\includegraphics[width=\columnwidth]{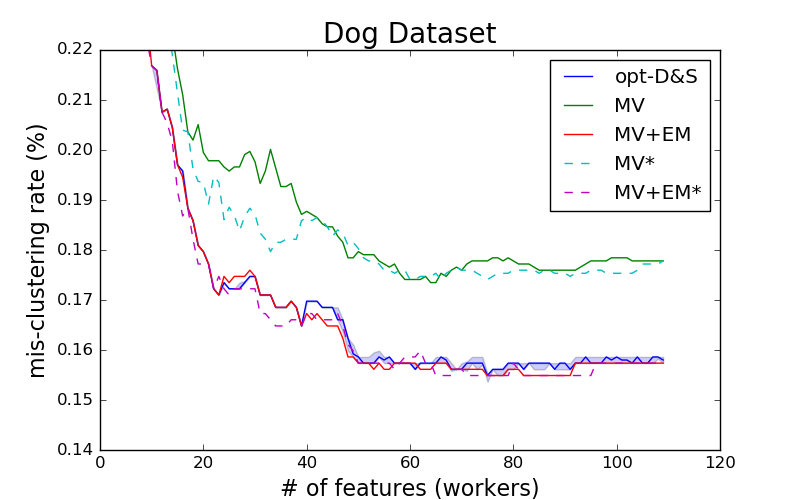}}
	\centerline{\includegraphics[width=\columnwidth]{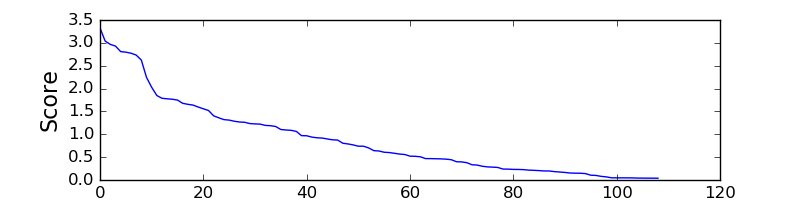}}
\end{minipage}
\begin{minipage}{.49\columnwidth}
	\centerline{\includegraphics[width=\columnwidth]{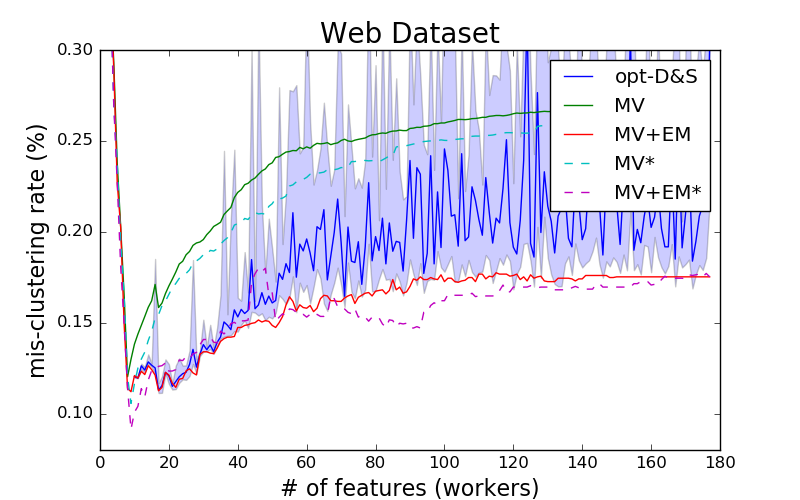}}
	\centerline{\includegraphics[width=\columnwidth]{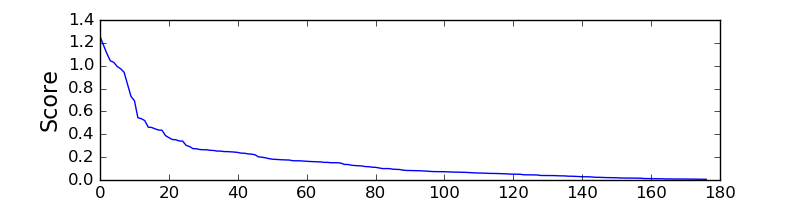}}
\end{minipage}
\begin{minipage}{.49\columnwidth}
	\caption{The top figure shows the mis-clustering rate of different algorithms under feature selection. For \emph{opt-D\&S}, the algorithm was repeated 20 times and the median is plotted while the first and the third quartiles are displayed as the shaded error bar. The dashed lines are benchmarks by utilizing the supervised feature selection mentioned in context. The bottom figure shows the mutual information score (equation \ref{eq:score}).}
	\label{fig:datasets}
\end{minipage}

\end{figure}


\textbf{The real datasets are redundant.} From all the figures, it is clear that there is a big drop in the mis-clustering rate when the top few features are utilized. As the curve gets flattened quickly, the marginal utility is diminishing fast.




\textbf{In most cases, the proposed feature selection technique (the solid lines) remains competitive, compared to the supervised feature selection (the dashed lines).} For example, for bird and dog datasets, the performance of our feature selection technique stays close to that of the supervised feature selection, especially when the number of features is small.


\textbf{Feature selection makes algorithms more robust and can potentially improve the outcomes.} We noticed that in some cases (e.g. TREC dataset and web dataset) the mis-clustering rate of $opt-D\&S$ fluctuates a lot. It is possibly because of the noise in the data. Feature selection helps filtering out noisy data and makes $opt-D\&S$ more robust. For web dataset, feature selection significant improves the error rates for all the algorithms.


\section{Discussion}
In this paper, we proposed a novel feature selection technique for learning MDPD which is based on pairwise mutual information. The utilization of mutual information was justified by a goodness-of-fit measure of the mixture model. Empirical studies of feature selection in application of crowdsourcing are also reported. Our feature selection algorithm are able to identify relevant, useful and informative features for MDPD, filters out the noise in the data, and makes the learning more robust. We argue that this feature selection technique is generic. It is not ad hoc for crowdsourcing, as it does not require any additional assumptions. Since it is based on mutual information, it is invariant to label swapping.



\bibliography{cite}
\bibliographystyle{plain}




\end{document}